\documentclass[letterpaper]{article} 
\usepackage{aaai24}  
\usepackage{times}  
\usepackage{helvet}  
\usepackage{courier}  
\usepackage[hyphens]{url}  
\usepackage{graphicx} 
\usepackage{natbib}  
\usepackage{caption} 
\usepackage{algorithm}
\usepackage{algorithmic}

\usepackage{amsthm}
\usepackage{amssymb}
\usepackage{amsmath}
\usepackage{bm}
\usepackage{multirow}
\usepackage{color}
\usepackage{subcaption}
\usepackage{booktabs}

\newtheorem{theorem}{Theorem}
\newtheorem{lemma}[theorem]{Lemma}

\newtheorem{definition}{Definition}

\usepackage{makecell}
\newtheorem{manualtheoreminner}{Theorem}
\newenvironment{manualtheorem}[1]{%
  \renewcommand\themanualtheoreminner{#1}%
  \manualtheoreminner
}{\endmanualtheoreminner}

\usepackage{newfloat}
\usepackage{listings}
\DeclareCaptionStyle{ruled}{labelfont=normalfont,labelsep=colon,strut=off}
\floatstyle{ruled}
\newfloat{listing}{tb}{lst}{}
\floatname{listing}{Listing}
\title{Towards Modeling Uncertainties of Self-explaining Neural Networks via Conformal Prediction}
\author{
    Wei Qian\equalcontrib\textsuperscript{\rm 1},
    Chenxu Zhao\equalcontrib\textsuperscript{\rm 1},
    Yangyi Li\equalcontrib\textsuperscript{\rm 1},
    Fenglong Ma\textsuperscript{\rm 2},
    Chao Zhang\textsuperscript{\rm 3},
    Mengdi Huai\textsuperscript{\rm 1}
}
\affiliations{

    \textsuperscript{\rm 1}Iowa State University\\
    \textsuperscript{\rm 2}Pennsylvania State University\\
    \textsuperscript{\rm 3}Georgia Institute of Technology\\

    \{wqi, cxzhao, liyangyi,  mdhuai\}@iastate.edu, fenglong@psu.edu, chaozhang@gatech.edu
}

\begin{document}

\maketitle

\begin{abstract}
Despite the recent progress in deep neural networks (DNNs), it remains challenging to explain the predictions made by DNNs. Existing explanation methods for DNNs mainly focus on post-hoc explanations where another explanatory model is employed to provide explanations. The fact that post-hoc methods can fail to reveal the actual original reasoning process of DNNs raises the need to build DNNs with built-in interpretability. Motivated by this, many self-explaining neural networks have been proposed to generate not only accurate predictions but also clear and intuitive insights into why a particular decision was made. However, existing self-explaining networks are limited in providing distribution-free uncertainty quantification for the two simultaneously generated prediction outcomes (i.e., a sample's final prediction and its corresponding explanations for interpreting that prediction). Importantly, they also fail to establish a connection between the confidence values assigned to the generated explanations in the interpretation layer and those allocated to the final predictions in the ultimate prediction layer. To tackle the aforementioned challenges, in this paper, we design a novel uncertainty modeling framework for self-explaining networks, which not only demonstrates strong distribution-free uncertainty modeling performance for the generated explanations in the interpretation layer but also excels in producing efficient and effective prediction sets for the final predictions based on the informative high-level basis explanations. We perform the theoretical analysis for the proposed framework. Extensive experimental evaluation demonstrates the effectiveness of the proposed uncertainty framework.
\end{abstract}

\section{Introduction}
\label{sec:intro}

Recent advancements in DNNs have undeniably led to remarkable improvements in accuracy. However, this progress has come at the expense of losing visibility into their internal workings \cite{nam2020relative}. This lack of transparency often renders these deep learning models as ``black boxes'' \cite{li2020mri,bang2021explaining,zhao2023automated}. Without understanding the rationales behind the predictions, these black-box models cannot be fully trusted and widely applied in critical areas. In addition, explanations can facilitate model debugging and error analysis. These indicate the necessity of investigating the explainability of DNNs.

To offer valuable insights into how a model makes predictions, many approaches have surfaced to explain the behavior of black-box neural networks. These methods primarily focus on providing post-hoc explanations, attempting to provide insights into the decision-making process of the underlying model \cite{lundberg2017unified,ribeiro2016should,slack2020fooling,slack2021reliable,yao2022concept,han2022explanation,huai2020towards}. Note that these post-hoc methods need designing another explanatory model to provide explanations for a trained deep learning model. Despite their widespread use, post-hoc interpretation methods can be inaccurate or incomplete in revealing the actual reasoning process of the original model \cite{rudin2018please}.

To address these challenges, researchers have been actively promoting the adoption of inherently self-explaining neural networks \cite{koh2020concept,havasi2022addressing,chauhan2023interactive,yuksekgonul2022post,alvarez2018towards,sinha2023understanding}, which integrate interpretability directly into the model architecture and training process, producing explanations that are utilized during classification. Based on the presence of ground-truth basis concepts during training, existing works on self-explaining networks can be broadly classified into two groups: those with access to ground-truth and those without. For example, \cite{koh2020concept} trains the concept bottleneck layer to align concept predictions with established true concepts, while \cite{alvarez2018towards} utilizes autoencoder techniques to learn prototype-based concepts from the training data.

Although existing self-explaining networks can offer inherent interpretability to make their decision-making process transparent and accessible, they are limited in the distribution-free uncertainty quantification for both the generated explanation outcomes and the corresponding label predictions. Such limitations have the potential to significantly impede the widespread acceptance and utilization of self-explaining neural networks. In practice, such confidence values serve as indicators of the likelihood of correctness for each prediction, and understanding the likelihood of each prediction allows us to assess the degree to which we can rely on it. Note that self-explaining neural networks simultaneously generate two outputs (i.e., the final prediction decision and an explanation of that final decision). The most straightforward way is to separately and independently quantify uncertainties for the predicted explanations and their corresponding final predictions. Nevertheless, we argue that within self-explaining neural networks, the high-level interpretation space is informative due to its powerful inductive bias in generating high-level basis concepts for interpretable representations. Imagine a scenario in which we engage in the binary classification task distinguishing between zebras and elephants, and we hold a strong 95\% confidence in identifying the presence of the ``Stripes'' concept within a test image sample. In this context, we can assert that the model exhibits high confidence in assigning the ``Zebra'' label to this particular test sample. This also aligns with existing class-wise explanations \cite{li2021instance,zhao2021deep} that focus on uncovering the informative patterns that influence the model's decision-making process when classifying instances into a particular class category.

Our goal in this paper is to offer distribution-free uncertainty quantification for self-explaining networks via two indispensable and interconnected aspects: establishing a reliable uncertainty measure for the generated explanations and effectively transferring the explanation prediction sets to the associated predicted class outcomes. Note that the recent work of \cite{slack2021reliable} focuses exclusively on post-hoc explanations and relies on certain assumptions regarding the properties of these explanations, such as assuming a Gaussian distribution of explanation errors. On the other hand, \cite{kim2023probabilistic} mainly presents probabilities and uncertainties for individual concepts or classes, and does not provide the prediction set with uncertainty, neither for the concept layer nor for the prediction layer. Therefore, in addition to providing distribution-free uncertainty quantification for the generated basis explanations, it is also essential to study how to transfer the constructed prediction sets in the interpretation space to the final prediction decision space.

To address the above challenges, in this paper, we design a novel \textbf{\underline{un}}certainty modeling framework for \textbf{\underline{s}}elf-\textbf{\underline{e}}xplaining \textbf{\underline{n}}eural \textbf{\underline{n}}etworks (\textbf{unSENN}), which can make valid and reliable predictions or estimates with quantifiable uncertainty, even in situations where the underlying data distribution may not be well-defined or where traditional assumptions of parametric statistics might not hold. Specifically, in our proposed method, we first design novel non-conformity measures to capture the informative high-level concepts in the interpretation layer for self-explaining neural networks. Importantly, our proposed non-conformity measures can overcome the difficulties posed by self-explaining networks, even in the absence of explanation ground truth in the interpretation layer. Note that these calculated scores reflect the prediction error of the underlying explainer, where a smaller prediction error would lead to the construction of smaller and more informative concept sets. We also prove the theoretical guarantees for the constructed concept prediction sets in the interpretation layer. To get better conformal predictions for the final predicted labels, we then design effective transfer functions, which involve the search of possible labels under certain constraints. Due to the challenge posed by directly estimating final prediction sets through the formulated transfer functions, we transform them into new optimization frameworks from the perspective of adversarial attacks. We also perform the theoretical analysis for the final prediction sets. The extensive experimental results verify the effectiveness of our proposed method.
\section{Problem Formulation}
Let us consider a general multi-class classification problem with a total of $M > 2$ possible labels. For an input $x \in \mathbb{R}^{d}$, its true label is represented as $y \in \mathbb{R}$. Note that self-explaining networks generate high-level concepts in the interpretation layer prior to the output layer. In Definition \ref{def:Self-explaining}, we give a general definition of a self-explaining model \cite{alvarez2018towards,elbaghdadi2020self}.

\begin{definition} [Self-explaining Models] \label{def:Self-explaining}
    Let $x \in \mathcal{X} \subset \mathbb{R}^{d}$ and $\mathcal{Y} \subset \mathbb{R}^{M}$ be the input and output spaces. We say that $f: \mathcal{X} \rightarrow \mathcal{Y}$ is a self-explaining prediction model if it has the following form
    \begin{align}
        & f(x)= g(\theta_{1}(x)h_{1}(x), \cdots, \theta_{C}(x)h_{C}(x)),
    \end{align}
    where the prediction head $g$ depends on its arguments, the set $\{(\theta_{c}(x),h_{c}(x))\}_{c=1}^{C}$ consists of basis concepts and their influence scores, and $C$ is small compared to $d$.
\end{definition}

\noindent \textbf{Example 1: Concept bottleneck models} \cite{koh2020concept}.  Concept bottleneck models (CBMs) are interpretable neural networks that first predict labels for human-interpretable concepts relevant to the prediction task, and then predict the final label based on the concept label predictions. We denote raw input features by $x$, the labels by $y$, and the pre-defined true concepts by $c \in \mathbb{R}^{C}$. Given the observed training samples $D^{tra}=\{(x_i,c_i,y_i)\}_{i=1}^{N^{tra}}$, a CBM \cite{koh2020concept} can be trained using the concept and class prediction losses via two distinct approaches: \emph{independent} and \emph{joint}.

\noindent \textbf{Example 2: Prototype-based self-explaining networks} \cite{alvarez2018towards}.
    Different from CBMs \cite{koh2020concept}, prototype-based self-explaining networks usually use an autoencoder to learn a set of prototype-based concepts directly from the training data (i.e., $D^{tra}=\{(x_i,y_i)\}_{i=1}^{N^{tra}}$) during the training process without the pre-defined concept supervision information. These prototype-based concepts are extracted in a way that they can best represent specific target sets \cite{zhang2023learning}.

For a well-trained self-explaining network $f=g \circ h$, our goal in this paper is to quantify the uncertainty of the generated explanations in the interpretation layer (preceding the output layer) and the final predictions in the output layer. In particular, for a test sample $x$ and a chosen miscoverage rate $\varepsilon \in (0,1)$, our primary objective is twofold: First, we aim to calculate a set $\Gamma^{\varepsilon}_{cpt}(x)$ containing $x$'s true concepts with probability at least $1-\varepsilon$. In addition, we will investigate how to transfer the explanation prediction sets to the final decision sets. In practice, predictions associated with confidence values are highly desirable in risk-sensitive applications.

\section{Methodology}
\label{sec:method}
In this section, we take concept bottleneck models \cite{koh2020concept,havasi2022addressing,chauhan2023interactive,yuksekgonul2022post} as an illustrative example to present our proposed method. Note that concept bottleneck models incorporate pre-defined concepts into a supervised learning procedure. We denote the available dataset as $D=\{(x_i,c_{i}, y_i)\}_{i=1}^{N}$, where $x_i \in \mathbb{R}^{d}$, $c_i=[c_{i,1},\cdots,c_{i,C}] \in \mathbb{R}^{C}$ and $y_{i}\in [M]$. These samples are drawn exchangeably from some unknown distribution $P_{XCY}$. To train a self-explaining network, we first split the available dataset $D=\{(x_i,c_{i}, y_i)\}_{i=1}^{N}$ into a training set $D^{tra}$ and a calibration set $D^{cal}$, where $D^{tra} \cap D^{cal} = \emptyset$ and $D^{tra} \cup D^{cal} = D$. Given the available training set $D^{tra}$ associated with the supervised concepts, we can train a self-explaining network $f=g \circ h$, where $h: \mathbb{R}^{d} \rightarrow \mathbb{R}^{C}$ maps a raw sample $x$ into the concept space and $g: \mathbb{R}^{C} \rightarrow \mathbb{R}^{M}$ maps concepts into a final class prediction.

Given the underlying well-trained self-explaining model $f=g \circ h$, our goal is to provide validity guarantees on its explanation and final prediction outcomes in a \emph{post-hoc} way. Note that each $x_{i}$, its true relevant concepts are represented by a concept vector $c_i=[c_{i,1},\cdots, c_{i,j}, \cdots, c_{i,C}] \in \{0,1\}^{C}$, with $c_{i,j}=1$ indicating that $x_{i}$ is associated with the $j$-th concept. We also use $\mathcal{C}_{cpt}(x_{i})=\{j | c_{i,j}=1\}$ to represent the set of true concepts for $x_{i}$. For $x_{i}$, we can use $h: \mathbb{R}^{d} \rightarrow \mathbb{R}^{C}$ to obtain $h(x_{i})=[h_{1}(x_{i}),\cdots,h_{j}(x_{i}),\cdots,h_{C}(x_{i})]$, where $h_{j}(x_{i})$ is the prediction score of $x_{i}$ with regards to the $j$-th concept. We denote $[h_{[1]}(x_{i}),\cdots,h_{[j]}(x_{i}),\cdots,h_{[C]}(x_{i})]$ as the sorted values of $h(x_{i})$ in the descending order, i.e., $h_{[1]}(x_{i})$ is the largest (top-1) score, $h_{[2]}(x_{i})$ is the second largest (top-2) score, and so on. Furthermore, $[j]$ corresponds to the concept index of the top-$j$ concept prediction score, i.e., $j'=[j]$ if $h_{j'}(x_{i})=h_{[j]}(x_{i})$. In practice, researchers are usually interested in identifying the top-$K$ important concepts to help users understand the key factors driving the model's predictions \cite{agarwal2022openxai,brunet2022implications,yeh2020completeness,huai2022towards,rajagopal2021selfexplain,hitzler2022human}. Here, we can easily convert a general concept predictor $h$ into a top-$K$ concept classifier by returning the set of concepts corresponding to the set of top-$K$ prediction scores from $h$. For $x_{i}$, the top-$K$ concept prediction function returns the set $\tilde{\mathcal{C}}_{cpt}(x_i)=\{[1],\cdots,[K]\}$ for $1 \leq K < C$. Here, we assume the concept prediction scores are calibrated, i.e., taking values in the range of $[0,1]$. For $\Psi \in \mathbb{R}$, we can use simple transforms such as $\sigma(\Psi)=1/(1+e^{-\Psi})$ to map it to the range of $[0,1]$ without changing their ranking. Then, for $x_{i}$, we obtain the below calibrated concept importance scores
\begin{align}
    &\tilde{h}(x_{i})=[1/(1+\exp(-h_{1}(x_{i}))),\cdots,1/(1 \\
    &\quad +\exp(-h_{C}(x_{i})))] = [\sigma(h_{1}(x_{i})),\cdots,\sigma(h_{C}(x_{i}))], \notag 
\end{align} 
\noindent where $\sigma(h_{j}(x_{i})) \in [0,1]$ is the calibrated importance score of $x_{i}$ with regards to the $j$-th concept. Not that for $x_{i}$, it is associated with a concept subset $\mathcal{C}_{cpt}(x_{i})$, which can be represented by a vector $h^{*}(x_{i})=[h^{*}_{1}(x_{i}), \cdots, h_{j}^{*}(x_{i}), \cdots, h^{*}_{C}(x_{i})]$, where $h_{j}^{*}(x_{i})=1$ if and only if concept $j$ is associated with $x_{i}$, and $0$ otherwise. 

In order to quantify how ``strange'' the generated concept explanations are for a given sample $x_{i}$, we utilize the underlying well-trained self-explaining model $f$ to construct the following non-conformity measure
\begin{align}
\label{eq:NonScoreFConcept}
    & s(x_i,h^{*}(x_{i}),\tilde{h}(x_{i}))= \sum_{j=1}^{K} (1- 1/(1+\exp(-h_{[j]}(x_{i})))) \notag\\
    & + \lambda_{2} * \sum_{j=K+1}^{C} (1/(1+\exp(-h_{[j]}(x_{i}))))= \sum_{j=1}^{K} |h_{[j]}^{*}(x_{i}) \notag\\
    &\quad\quad -\tilde{h}_{[j]}(x_{i})|  + \lambda_{2} * \sum_{j=K+1}^{C} |h_{[j]}^{*}(x_{i}) - \tilde{h}_{[j]}(x_{i}) |, 
\end{align}
\noindent where $\lambda_{2}$ is a pre-defined trade-off parameter, and $[j]$ corresponds to the concept index of the top-$j$ concept prediction score of $\tilde{h}(x_{i})$. The above non-conformity measure returns small values when the concept prediction outputs of the true relevant concepts are high and the concept outputs of the non-relevant concepts are low. Note that smaller conformal scores correspond to more model confidence \cite{angelopoulos2021gentle,teng2021t}. In particular, when $\lambda_{2}=0$, we only consider the top-ranked concepts.

Based on the above constructed non-conformity measure, for each calibration sample $(x_{i},c_{i},y_{i}) \in D^{cal}$, we apply the non-conformity measure $s(x_i,h^{*}(x_{i}),\tilde{h}(x_{i}))$ to get $N^{cal}$ non-conformity scores $\{s_{i}\}_{i=1}^{N^{cal}}$, where $N^{cal}$ is the number of calibration samples. Then, we select a miscoverage rate $\varepsilon \in (0,1)$ and compute $Q_{1-\varepsilon}$ as the $\frac{\lceil (N^{cal}+1)(1-\varepsilon) \rceil}{N^{cal}}$ quantile of the non-conformity scores $\{s_{i}\}_{i=1}^{N^{cal}}$, where $\lceil \cdot \rceil$ is the ceiling function. Formally, $Q_{1-\varepsilon}$ is defined as
\begin{align}
\label{eq:ScoreFConcept00}
    & Q_{1-\varepsilon} = \inf \{Q: \frac{|\{i: s(x_i,h^{*}(x_{i}),\tilde{h}(x_{i})) \leq Q \}| }{N^{cal}} \\
    & \qquad\qquad\qquad \geq \frac{\lceil (N^{cal}+1)(1-\varepsilon) \rceil}{N^{cal}} \}. \notag 
\end{align}
\noindent Based on the above, given $Q_{1-\varepsilon}$ and a desired coverage level $1-\varepsilon \in (0,1)$, for a test sample $x^{test}$ (where $x^{test}$ is known but the true top-$K$ concepts $\mathcal{C}$ are not), we can get the explanation prediction set $\Gamma^{\varepsilon}_{cpt}(x^{test})$ for $x^{test}$ with $1-\varepsilon$ coverage guarantee. Algorithm \ref{alg:1} gives the procedure to calculate the concept prediction sets.

\begin{algorithm}[!t]
    \caption{Uncertainty quantification for self-explaining neural networks}
    \label{alg:1}
    \begin{algorithmic}[1]
        \STATE \textbf{Input:} Calibration data $D^{cal}$, pre-trained model $f=g \circ h$, non-conformity measure $s(\cdot)$, a test sample $x^{test}$, and a desired miscoverage rate $\varepsilon \in (0,1)$.
        \STATE \textbf{Output:} The concept prediction set $\Gamma^{\varepsilon}_{cpt}(x^{test})$ and label set $\Gamma^{\varepsilon}_{lab}(x^{test})$ for sample $x^{test}$.
        \STATE Calculate the non-conformity scores on the calibration samples (i.e., $D^{cal}$) based on Eqn. (\ref{eq:NonScoreFConcept})
        \STATE Compute the quantile value $Q_{1-\varepsilon}$ based on Eqn. (\ref{eq:ScoreFConcept00})
        \STATE Construct the concept prediction set $\Gamma^{\varepsilon}_{cpt}(x^{test})$ for $x^{test}$ with $1-\varepsilon$ coverage guarantee 
        \STATE Initialize $\Gamma^{\varepsilon}_{lab}(x^{test})=\emptyset$
        \FOR {$m=1, 2, \cdots, M$}
            \STATE Obtain $\delta$ by optimizing the loss in Eqn. (\ref{eq:DeltaOptima})
            \IF{$\sum_{m' \in [M]\setminus m} \mathbb{I} [g_{m} (v) > g_{m'}(v) ]==M-1$}
            \STATE Update $\Gamma^{\varepsilon}_{lab}(x^{test})=\Gamma^{\varepsilon}_{lab}(x^{test}) \cup \{m \}$
            \ENDIF
        \ENDFOR
        \STATE {\bfseries Return:} $\Gamma^{\varepsilon}_{cpt}(x^{test})$ and $\Gamma^{\varepsilon}_{lab}(x^{test})$.
    \end{algorithmic}
\end{algorithm}

\begin{definition} [Data Exchangeability] \label{def:DataExchange}
    Consider variables $z_{1},\cdots,z_{N}$. The variables are exchangeable if for every permutation $\tau$ of the integers $1,\cdots,N$, the variables $w_{1},\cdots,w_{N}$ where $w_{i}=z_{\tau(i)}$, have the same joint probability distribution as $z_{1},\cdots,z_{N}$.
\end{definition}

\begin{theorem} \label{thm:CptSetThm}
    Suppose that the calibration samples ($D^{cal}=\{(x_i,c_{i}, y_i)\}_{i=1}^{N^{cal}}$) and the given test sample $x^{test}$ are exchangeable. Then, if we calculate the quantile value $Q_{1-\varepsilon}$ and construct $\Gamma^{\varepsilon}_{cpt}(x^{test})$ as indicated above, for the above non-conformity score function $s$ and any $\varepsilon \in (0,1)$, the derived $\Gamma^{\varepsilon}_{cpt}(x^{test})$ satisfies
    \begin{align}
    \label{eq:ConceptSets}
    & P(\mathcal{C}_{cpt} (x^{test}) \in \Gamma^{\varepsilon}_{cpt}(x^{test}) ) \geq 1-\varepsilon,
    \end{align}
\noindent where $\mathcal{C}_{cpt} (x^{test})$ is the true relevant concept set for the given test sample $x^{test}$, and $\varepsilon$ is a desired miscoverage rate.
\end{theorem}
\begin{proof}
    Let $s_{1}, s_{2}, \cdots, s_{N^{cal}}, s_{test}$ denote the non-formality scores of the calibration samples and the given test samples. To avoid handling ties, we consider the case where these non-formality scores are distinct with probability 1. Without loss of generality, we assume these calibration scores are sorted so that $s_{1} < s_{2} < \cdots < s_{N^{cal}}$. In this case, we have that $Q_{1-\varepsilon}=s_{\lceil (N^{cal}+1)(1-\varepsilon) \rceil}$ when $\varepsilon \geq \frac{1}{N^{cal}+1}$ and $Q_{1-\varepsilon}=\infty$ otherwise. Note that in the case where $Q_{1-\varepsilon}=\infty$, the coverage property is trivially satisfied; thus, we only need to consider the case when $\varepsilon \geq \frac{1}{N^{cal}+1}$. We begin by observing the below equality of the two events
    \begin{align}
        & \{ \mathcal{C}_{cpt} (x^{test}) \in \Gamma^{\varepsilon}_{cpt}(x^{test})\}=\{s_{test} \leq Q_{1-\varepsilon} \}.
    \end{align}
\noindent By combining this with the definition of $Q_{1-\varepsilon}$, we have
    \begin{align}
        & \{ \mathcal{C}_{cpt} (x^{test}) \in \Gamma^{\varepsilon}_{cpt}(x^{test})\}=\{s_{test} \leq s_{\lceil (N^{cal}+1)(1-\varepsilon) \rceil}\}. \notag
    \end{align}
    \noindent Since the variables $s_{1},\cdots,s_{test}$ are exchangeable, we can have that $P(s_{test} \leq s_{k}) =k / (N^{cal}+1)$ holds true for any integer $k$. This means that $s_{test}$ is equally likely to fall in anywhere between the calibration scores $s_{1},\cdots,s_{N^{cal}}$. Notably, the randomness is over all variables. Building on this, we can derive that $P ( s_{test} \leq s_{\lceil (N^{cal}+1)(1-\varepsilon) \rceil} ) = \frac{\lceil (N^{cal}+1)(1-\varepsilon) \rceil}{(N^{cal}+1)} \geq 1-\varepsilon$, which yields the desired result.
\end{proof}
In the above theorem, we show that the proposed algorithm can output predictive concept sets that are rigorously guaranteed to satisfy the desired coverage property shown in Eqn. (\ref{eq:ConceptSets}), no matter what (possibly incorrect) self-explaining model is used or what the (unknown) distribution of the data is. The above theorem relies on the assumption that the data points are exchangeable (see Definition \ref{def:DataExchange}), which is widely adopted by existing conformal inference works \cite{cherubin2021exact,tibshirani2019conformal}. The data exchangeability assumption is much weaker than i.i.d. Note that i.i.d variables are automatically exchangeable \cite{sesia2022conformal}. For a random variable $v \in \mathbb{R}^{C}$, we denote $[v_{[1]},v_{[2]},\cdots,v_{[C]}]$ as the sorted values of $v$ in the descending order, i.e., $v_{[1]}$ is the largest (top-1) score, $v_{[2]}$ is the second largest (top-2) score, and so on. Then, we can construct the label prediction set for the test sample $x^{test}$ as
\begin{align}
\label{eq:NonScoreFLabel}
    & \Gamma^{\varepsilon}_{lab} (x^{test}) = \{y=g(v): \sum_{j=1}^{K} |h_{[j]}^{*}(x^{test})-\tilde{h}_{[j]}(v)|  \notag \\
    &  + \lambda_{2} * \sum_{j=K+1}^{C} |h_{[j]}^{*}(x^{test}) - \tilde{h}_{[j]}(v) | \leq Q_{1-\varepsilon}\},
\end{align}
\noindent where $\tilde{h}(v)=\sigma(v)$. In the above, since the true concepts are unknown for $x^{test}$, we instead calculate $h_{[j]}^{*}(x^{test})=\tilde{h}_{[j]}(x^{test})=\sigma(h_{[j]}(x^{test}))$. Without loss of generality, in the following, we take a setting where $\lambda_{2}=1$ as an example to discuss how to transfer the confidence prediction sets in the concept space to the final class label prediction layer, and the \emph{extension} to other cases can be easily derived. In this case, we calculate the set for $x^{test}$ as $\Gamma^{\varepsilon}_{lab} (x^{test})  = \{y=g(v): ||\tilde{h}(v)-h^{*}(x^{test})||_{1} \leq Q_{1-\varepsilon}\}$, where $v \in \mathbb{R}^{C}$, $\hat{v}=\sigma(v)=[\sigma(v_{1}),\cdots,\sigma(v_{C})]$, $g$ is the prediction head, $\tilde{h}(x^{test})=[\sigma(h_{1}(x^{test})),\cdots,\sigma(h_{C}(x^{test}))]$ and $Q_{1-\varepsilon}$ is derived based on the calibration set. To derive the label sets, we propose to solve the following optimization
\begin{align}
\label{eq:LabelSetOpt22}
    & \quad \max_{v \in \mathbb{R}^{C}} \sum_{m' \in [M]\setminus m} \mathbb{I} [g_{m}(v) > g_{m'} (v)] \\
    &\quad\quad s.t., ||\tilde{h}(v)-h^{*}(x^{test})||_{1} \leq Q_{1-\varepsilon},\notag 
\end{align}
\noindent where the quantile value $Q_{1-\varepsilon}$ is derived based on Eqn. (\ref{eq:ScoreFConcept00}). In the above, we want to find a $v$, whose predicted label is ``$m$'', under the constraint that $||\tilde{h}(v)-h^{*}(x^{test})||_{1} \leq Q_{1-\varepsilon}$. If we can successfully find such a $v$, we will include the label ``$m$'' in the label set; otherwise, we will exclude it. Notably, we need to iteratively solve the above optimization over all possible labels to construct the final prediction set.

However, it is difficult to directly optimize the above problem, due to the non-differentiability of the discrete components in the above equation. Below, we use $\delta=[\delta_{1},\cdots,\delta_{C}]$ to denote the difference between $\hat{v}$ and $\tilde{h}(x^{test})$, i.e., $\delta =\hat{v}-\tilde{h}(x^{test})$, where $\hat{v}=\sigma(v)=[\sigma(v_{1}),\cdots,\sigma(v_{C})]$, and $\tilde{h}(x^{test})=\sigma(h(x^{test}))$. Based on this, we can have
\begin{align}
\label{eq:VTransform}
    & v=[v_{1}=-\log (1/ (\tilde{h}_{1} (x^{test})+\delta_{1}) -1),\cdots, \\
    &\qquad\quad v_{C}=-\log (1/ (\tilde{h}_{C} (x^{test})+\delta_{C}) -1)]. \notag
\end{align}
\noindent To address the above challenge, we propose to solve the below reformulated optimization
\begin{align}
\label{eq:DeltaOptima}
  &  \min_{||\delta ||_{1} \leq Q_{1-\varepsilon}} \max (\max_{m'\neq m} g_{m'}(v) - g_{m}(v),-\beta),
\end{align}
\noindent where $m' \in [M]$, $\beta$ is a pre-defined value, $v$ is derived based on Eqn. (\ref{eq:VTransform}), and $g: \mathbb{R}^{C} \rightarrow \mathbb{R}^{M}$ maps $v$ into a final class prediction. $g(v)$ represents the logit output of the underlying self-explaining model for each class when the input is $v$. The above adversarial loss enforces the actual prediction of $v$ to the targeted label ``$m$'' \cite{huang2020metapoison,carlini2017towards}. If we can find such a $v$, the label ``$m$'' will be included in the prediction set; otherwise, it will not.

\begin{lemma} [\cite{tibshirani2019conformal}] \label{Lma:QuantileLma}
    If $V_{1}, \cdots, V_{n+1}$ are exchangeable random variables, then for any $\gamma \in (0,1)$, we have the following
    \begin{align}
        & P \{V_{n+1} \leq \text{Quantile} (\gamma; V_{1:n} \cup \{\infty\})\} \geq \gamma,
    \end{align}
    where $\text{Quantile} (\gamma; V_{1:n} \cup \{\infty\})$ denotes the level $\gamma$ quantile of $V_{1:n} \cup \{\infty\}$. In the above, $V_{1:n}=\{V_{1},\cdots,V_{n}\}$.
\end{lemma}

\begin{theorem} \label{thm:LabelSetThm}
    Suppose the calibration samples and the test sample are exchangeable, if we construct $ \Gamma^{\varepsilon}_{lab}(x^{test})$ as indicated above, the following inequality holds for any $\varepsilon \in (0,1)$
    \begin{align}
    & P(y^{test} \in \Gamma^{\varepsilon}_{lab} (x^{test}) ) \geq 1-\varepsilon,
    \end{align}
    where $y^{test}$ is the true label of $x^{test}$, and $\Gamma^{\varepsilon}_{lab} (x^{test})$ is the label prediction set derived based on Eqn. (\ref{eq:NonScoreFLabel}).
\end{theorem}

Due to space constraints, the complete proof for Theorem \ref{thm:LabelSetThm} is postponed to the Appendix. This proof is based on Lemma \ref{Lma:QuantileLma}.

\textbf{Discussions on prototype-based self-explaining neural networks.} Here, we will discuss how to generalize our above method to the prototype-based self-explaining networks, which do not depend on the pre-defined expert concepts \cite{alvarez2018towards,li2018deep}. They usually adopt the autoencoder network to learn a lower-dimension latent representation of the data with an encoder network, and then use several network layers over the latent space to learn $C$ prototype-based concepts, i.e., $\{p_{j}\in \mathbb{R}^{q}\}_{j=1}^{C}$. After that, for each $p_{j}$, the similarity layer computes its distance from the learned latent representation (i.e., $e_{enc}(x)$) as $h_{j}(x)=|| e_{enc}(x)-p_{j} ||_{2}^{2}$. The smaller the distance value is, the more similar $e_{enc}(x)$ and the $j$-th prototype ($p_{j}$) are. Finally, we can output a probability distribution over the $M$ classes. Nonetheless, we have no access to the true basis explanations if we want to conduct conformal predictions at the prototype-based interpretation level. To address this, for prototype-based self-explaining networks, we construct the following non-conformity measure
\begin{align}
\label{eq:NonconforScoreFAutoEn}
    & s(x_i,y_{i},f)=  \inf_{v \in \{v: \ell(g(v),y_{i}) \leq \alpha\} } [\sum_{j=1}^{K} |h_{[j]}(x_{i})-v_{[j]}| \notag \\
    &\qquad\qquad + \lambda_{2} * \sum_{j=K+1}^{C} |h_{[j]}(x_{i})-v_{[j]}| ],
\end{align}
\noindent where $\ell(g(v),y_{i})$ is the classification cross entropy loss for $v$, and $\alpha$ is a predefined small threshold value to make the prediction based on $v$ to be close to label $y_i$. Based on the above non-conformity measure, we can calculate the non-conformity scores for the calibration samples, and then calculate the desired quantile value. Based on this, for the given test sample, following Eqn. (\ref{eq:LabelSetOpt22}) and (\ref{eq:NonconforScoreFAutoEn}), we can calculate the final prediction set. However, it is usually complicated to calculate the above score due to the infimum operator. To overcome this challenge, we apply gradient descent starting from the trained similarity vector to find a surrogate vector around it. Additionally, to obtain the final prediction sets, we follow Eqn. (\ref{eq:DeltaOptima}) to determine whether a specific class label should be included in the final prediction sets.

\section{Experiments}
\label{sec:exp}
In this section, we conduct experiments to evaluate the performance of the proposed method (i.e., unSENN). Due to space limitations, more experimental details and results (e.g., experiments on more self-explaining models and running time) can be found in the Appendix.

\subsection{Experimental Setup}
\textbf{Real-world datasets.} In experiments, we adopt the following real-world datasets: \textbf{CIFAR-100 Super-class} \cite{fischer2019dl2} and \textbf{MNIST} \cite{deng2012mnist}. Specifically, the CIFAR-100 Super-class dataset comprises 100 diverse classes of images, which are further categorized into 20 super-classes. For example, the five classes \emph{baby, boy, girl, man} and \emph{woman} belong to the super-class \emph{people}. In this case, we exploit each image class as a concept-based explanation for the super-class prediction \cite{hong2023concept}. The MNIST dataset consists of a collection of grayscale images of handwritten digits (0-9). In experiments, we consider the MNIST range classification task, wherein the goal is to classify handwritten digits into five distinct classes based on their numerical ranges. Specifically, the digits are divided into five non-overlapping ranges (i.e., $[0,1], [2,3], [4,5],[6,7], [8,9]$), each assigned a unique label (the label set is denoted as $\{0, 1, 2, 3, 4\}$). For example, an input image depicting a handwritten digit ``4'' is classified as class 2 within the range of digits 4 to 5. Here, we exploit the identity of each digit as a concept-based explanation for the ranges prediction \cite{rigotti2021attention}. 

\begin{figure}[t]
\centering
\begin{subfigure}{0.495\linewidth}
\includegraphics[width=1\linewidth]{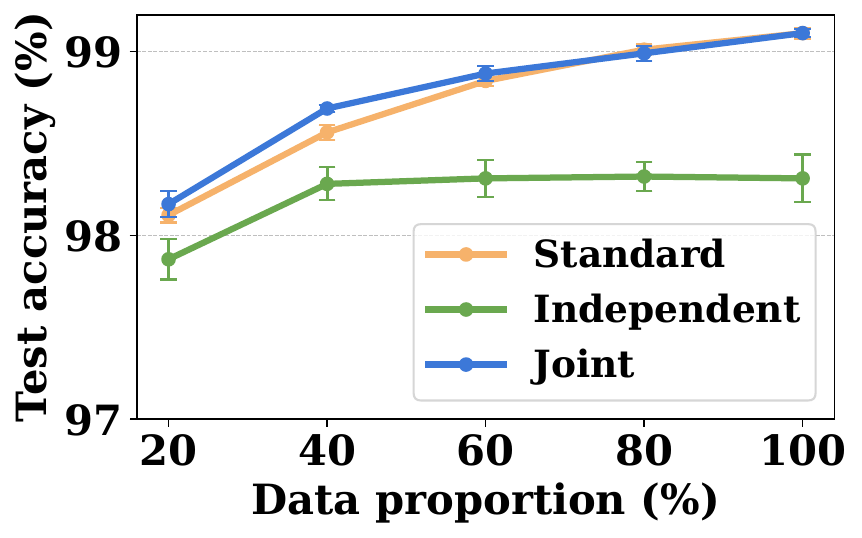}
\caption{MNIST}
\end{subfigure}
\begin{subfigure}{0.495\linewidth}
\includegraphics[width=1\linewidth]{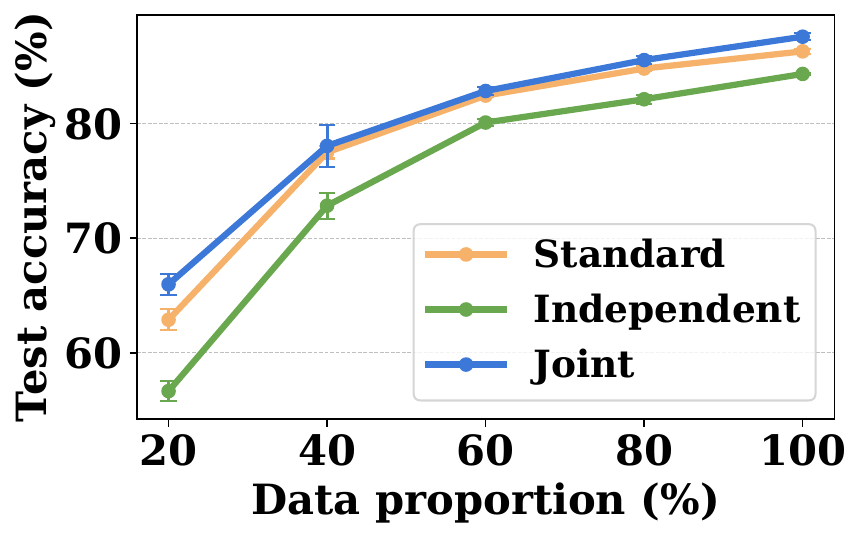}
\caption{CIFAR-100 Super-class}
\end{subfigure}
\caption{Data efficiency for concept bottleneck models with different learning strategies.}
\label{fig:data_efficiency}
\end{figure}

\begin{figure*}[t]
\centering
\begin{subfigure}{0.495\linewidth}
\includegraphics[width=1\linewidth]{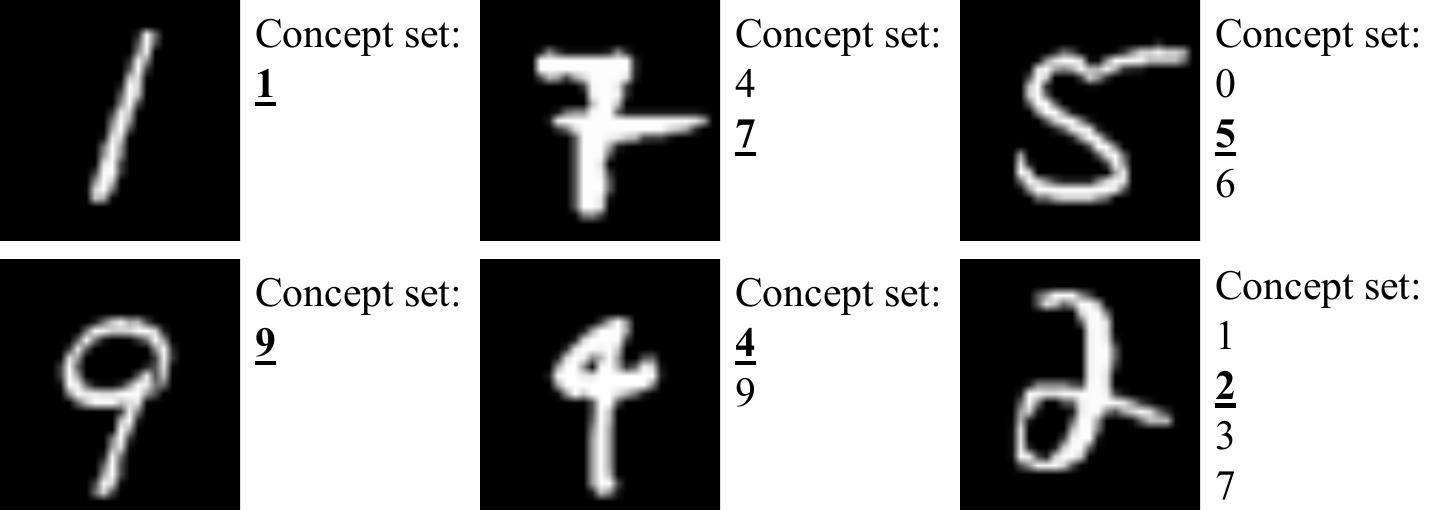}
\caption{MNIST}
\end{subfigure}
\begin{subfigure}{0.495\linewidth}
\includegraphics[width=1\linewidth]{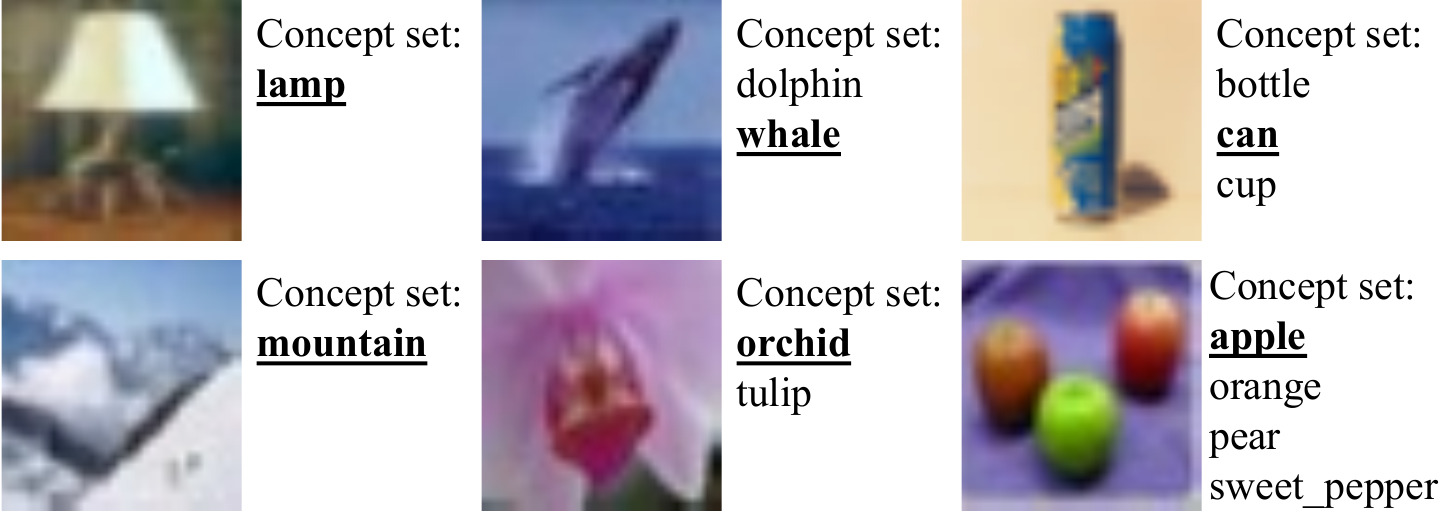}
\caption{CIFAR-100 Super-class}
\end{subfigure}
\caption{Concept conformal sets on MNIST and CIFAR-100 Super-class. Phrases in bold and underlined mean true concepts.}\label{fig:concept_conformal}
\end{figure*}

\textbf{Baselines.} In experiments, we adopt two baselines: \emph{Naive} baseline and \emph{Vanilla Conformal Predictors}. These baselines exclusively rely upon $\mathcal{X}$ and $\mathcal{Y}$. Specifically, the Naive baseline constructs the set sequentially incorporating classes in decreasing order of their probabilities until the cumulative probability surpasses the predefined threshold $1-\varepsilon$. Notably, this Naive approach lacks a theoretical guarantee of its coverage. The Vanilla Conformal Predictors utilize the final decisions to construct the prediction sets. Following \cite{cherubin2021exact}, we set the non-conformity measure for this Vanilla baseline as $s(x,y)=-f_{y}(x)$, where $f$ is a pre-trained model, and $f_{y}(x)$ is the softmax probability for  $x$ with true label $y$.

\textbf{Implementation details.} In experiments, we evaluate the proposed method (i.e., unSENN) across the following underlying models: ResNet-50 \cite{he2016deep}, a convolutional neural network (CNN), and a multi-layer perception (MLP). For the calibration set, we randomly hold out $10\%$ of the original available dataset to compute the non-conformity scores. All the experiments are conducted 10 times, and we report the mean and standard errors.

\begin{table}[t] 
\small
\centering
\begin{tabular}{c@{\hspace{10pt}}c@{\hspace{10pt}}cc} 
\toprule
\multirow{2}{*}{Dataset} & \multirow{2}{*}{\makecell{Miscoverage\\ rate $\varepsilon$}} & \multicolumn{2}{c}{Concept error rate} \\ 
\cmidrule(lr){3-4} & & unSENN & Naive\\ 
\midrule
\multirow{4}{*}{MNIST} & $0.05$ & $0.047 \pm 0.001 $ & $0.003 \pm 0.000 $\\
& $0.1$ & $0.094 \pm 0.002$ & $0.003 \pm 0.000 $ \\
& $0.15$ & $0.139 \pm 0.006$ & $0.004 \pm 0.000 $ \\
& $0.2$ & $0.174 \pm 0.011$ & $0.005 \pm 0.000 $\\
\midrule
\multirow{4}{*}{\makecell{CIFAR-100 \\ Super-class}} & $0.05$ & $0.044 \pm 0.002 $ & $0.073 \pm 0.001 $\\
& $0.1$ & $0.091 \pm 0.003$ & $0.100 \pm 0.001 $ \\
& $0.15$ & $0.137 \pm 0.004$ & $0.121 \pm 0.001 $ \\
& $0.2$ & $0.177 \pm 0.003$ & $0.136 \pm 0.001 $ \\
\bottomrule
\end{tabular}
\caption{Error rate of concept conformal sets on MNIST and CIFAR-100 Super-class with different miscoverage rates.}
\label{tab:error_rate}
\end{table}

\subsection{Experimental Results}
First, we measure the data efficiency (the amount of training data needed to achieve a desired level of accuracy) of the self-explaining model. For the CIFAR-100 Super-class, we apply ResNet-50 and MLP to learn the concepts and labels, respectively. Similarly, we use CNN and MLP for the MNIST. In Figure \ref{fig:data_efficiency}, we present the test accuracy of the self-explaining model using independent and joint learning strategies with various data proportions. Note that the standard model ignores the concepts and directly maps input features to the final prediction. As we can see, the self-explaining models are particularly effective on the adopted datasets. The joint model exhibits comparable performance to the standard model, and the joint model is slightly more accurate in lower data regimes (e.g., $20\%$). Therefore, in the following experiments, we adopt the joint learning strategy.

\begin{table}[t] 
\small
\centering
\begin{tabular}{cccc} 
\toprule
Dataset & Method & Label coverage & Average set size \\ 
\midrule
\multirow{3}{*}{MNIST} & unSENN & $0.988 \pm 0.002 $ & $1.00 \pm 0.00$\\
& Naive & $0.985 \pm 0.002$ & $1.01 \pm 0.00$\\
& Vanilla & $0.903 \pm 0.001$ & $0.90 \pm 0.00$\\
\midrule
\multirow{3}{*}{\makecell{CIFAR-100 \\ Super-class}} & unSENN & $0.847 \pm 0.002$ & $1.03 \pm 0.00$\\
& Naive & $0.786 \pm 0.004$ & $1.28 \pm 0.02$\\
& Vanilla & $0.805 \pm 0.006$ & $1.15 \pm 0.01$\\
\bottomrule
\end{tabular}
\caption{Label coverage and average set size of prediction conformal sets on MNIST and CIFAR-100 Super-class.}
\label{tab:coverage}
\end{table}

\begin{figure}[t]
\centering
\begin{subfigure}{0.485\linewidth}
\includegraphics[width=1\linewidth]{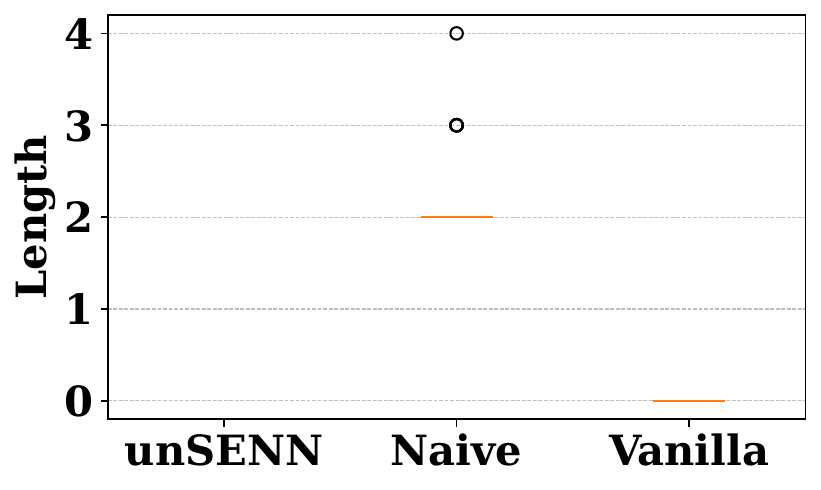}
\caption{MNIST}
\end{subfigure}
\begin{subfigure}{0.499\linewidth}
\includegraphics[width=1\linewidth]{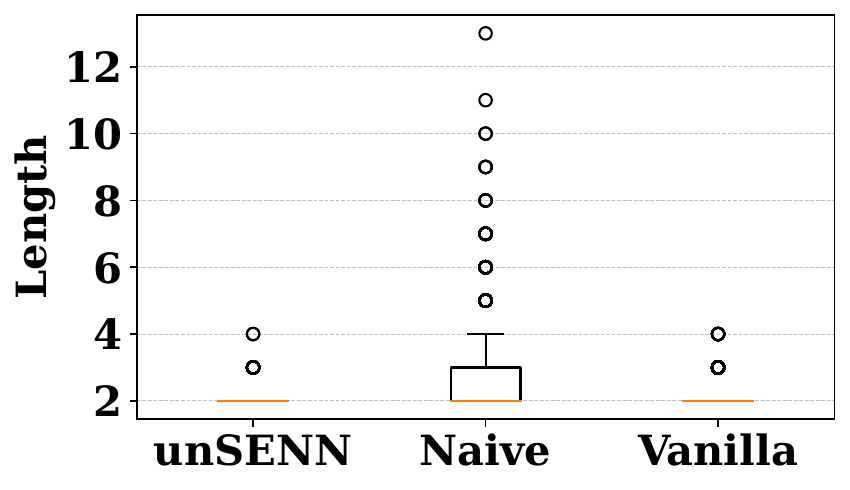}
\caption{CIFAR-100 Super-class}
\end{subfigure}
\caption{Bloxplot of prediction sets. We display the lengths of conformal sets that are not equal to 1 over all the test data.}
\label{fig:box_plot}
\end{figure}

\begin{figure*}[t]
\centering
\includegraphics[width=1\linewidth]{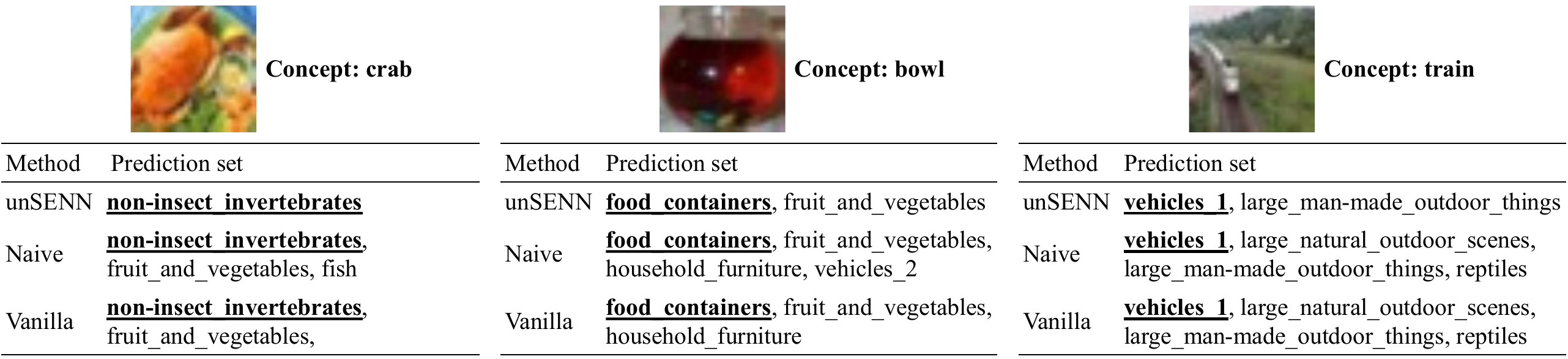}
\caption{Prediction conformal sets on CIFAR-100 Super-class. Phrases in bold and underlined mean true labels.}\label{fig:label_conformal}
\end{figure*}

Then, we explore the effectiveness of the concept conformal sets within the self-explaining model. In Table \ref{tab:error_rate}, we report the error rate (the percentage of true concepts that are not held in the prediction sets) of concept conformal sets with different miscoverage rates. From this table, we can find that our proposed method can guarantee a probability of error of at most $\varepsilon$ and fully exploits its error margin to optimize efficiency. In contrast, the Naive baseline simply outputs all possible concepts using the raw softmax probabilities and does not provide theoretical guarantees. In addition, we show concept conformal sets when $\varepsilon=0.1$ in Figure \ref{fig:concept_conformal}. Intuitively, a blurred image tends to have a large prediction set due to its uncertainty during testing, while a clear image tends to have a small set. Thus, a smaller set means higher confidence in the prediction, and a larger set indicates higher uncertainty regarding the predicted concepts.

Next, we investigate the effectiveness of the final prediction sets within the self-explaining model. We provide insights into the label coverage (the percentage of times the prediction set size is 1 and contains the true label) as well as the average set size. Table \ref{tab:coverage} reports the obtained experimental results when $\varepsilon=0.1$. From this table, we can see that our proposed method consistently outperforms both the Naive and Vanilla baselines. Specifically, our proposed method achieves notably higher label coverage and much tighter confidence sets. A more comprehensive understanding of the conformal set characteristics is illustrated in Figure \ref{fig:box_plot}, where we display the lengths of conformal sets that are not equal to 1. We can find that the Vanilla baseline contains empty sets on MNIST, and the Naive baseline includes much looser prediction sets on the CIFAR-100 Super-class. In contrast, our proposed method generates reliable and tight prediction sets. Additionally, in Figure \ref{fig:label_conformal}, we demonstrate prediction conformal sets when $\varepsilon=0.1$. We can observe that all the methods correctly output the true label, while the prediction set of our proposed method is the tightest. From these reported experimental results, we can conclude that our proposed method effectively produces uncertainty prediction sets for the final predictions, leveraging the informative high-level basis concepts.

\begin{figure}[t]
\centering
\begin{subfigure}{0.5\linewidth}
\includegraphics[width=1\linewidth]{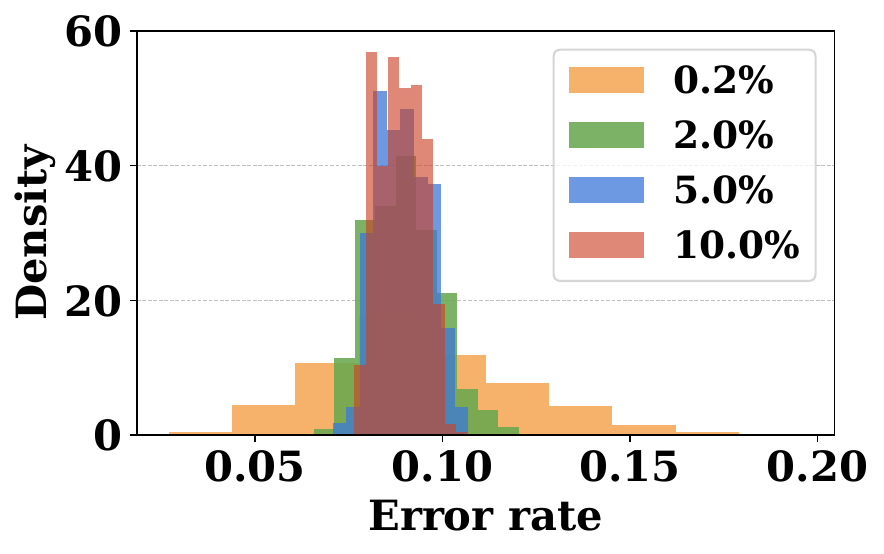}
\caption{MNIST}
\end{subfigure}
\begin{subfigure}{0.48\linewidth}
\includegraphics[width=1\linewidth]{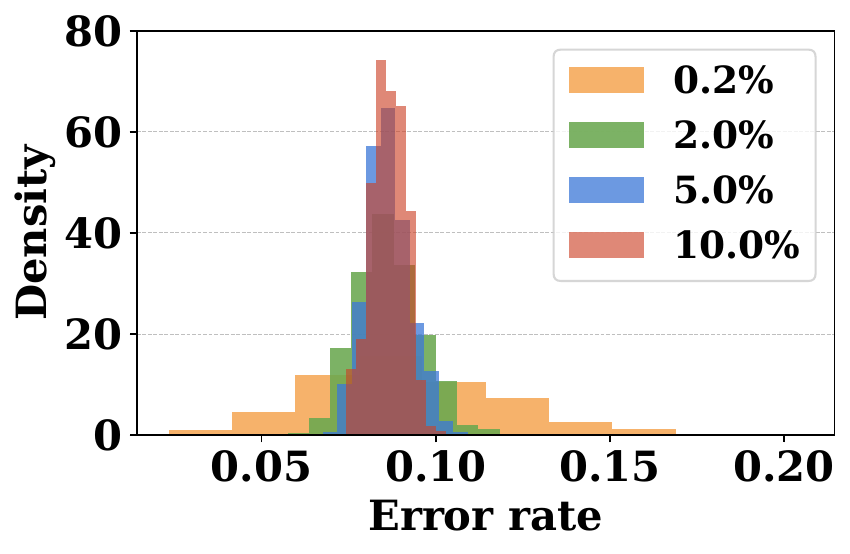}
\caption{CIFAR-100 Super-class}
\end{subfigure}
\caption{Distribution of error rate with different calibration sizes (splitting percentage of the original available data).}
\label{fig:cal_dis}
\end{figure}

Further, we perform an ablation study concerning the calibration data size. In Figure \ref{fig:cal_dis}, we illustrate the error rate distribution for the concept conformal sets with a miscoverage rate $\varepsilon=0.1$ on MNIST and CIFAR-100 Super-class, using different calibration data sizes. It is evident that utilizing a larger calibration set can enhance the stability of the prediction set. Therefore, we incorporate $10\%$ of the original available dataset for calibration without tuning, which also works well in our experiments.

\section{Related Work}
Currently, many self-explaining networks have been proposed \cite{koh2020concept,havasi2022addressing,chauhan2023interactive,yuksekgonul2022post,alvarez2018towards,li2018deep,kim2023probabilistic,zarlenga2023tabcbm}. However, they fail to provide distribution-free uncertainty quantification for the two simultaneously generated prediction outcomes in self-explaining networks. On the other hand, to tackle the uncertainty issues, traditional methods (e.g., Bayesian approximation) for establishing confidence in prediction models usually make strict assumptions and have high computational cost properties. In this work, we build upon conformal prediction \cite{gibbs2021adaptive,chernozhukov2018exact,xu2023sequential,tibshirani2019conformal,teng2022predictive,martinez2023approximating,ndiaye2022stable,lu2022fair,jaramillo2021shapley,liu2022conformalized}, which do not require strict assumptions on the underlying data distribution and has been applied before to classification and regression problems to attain marginal coverage. However, our work is different from existing conformal works in that we first model explanation prediction uncertainties in both scenarios – with and without true explanations. We also transfer the explanation uncertainties to the final prediction uncertainties through designing new transfer functions and associated optimization frameworks for constructing the final prediction sets.

\section{Conclusion}
In this paper, we design a novel uncertainty quantification framework for self-explaining networks, which generates distribution-free uncertainty estimates for the two interconnected outcomes (i.e., the final predictions and their corresponding explanations) without requiring any additional model modifications. Specifically, in our approach, we first design the non-conformity measures to capture the relevant basis explanations within the interpretation layer of self-explaining networks, considering scenarios with and without true concepts. Then, based on the designed non-conformity measures, we generate explanation prediction sets whose size reflects the prediction uncertainty. To get better label prediction sets, we also develop novel transfer functions tailored to scenarios where self-explaining networks are trained with and without true concepts. By leveraging the adversarial attack paradigm, we further design novel optimization frameworks to construct final prediction sets from the designed transfer functions. We provide the theoretical analysis for the proposed method. Comprehensive experiments validate the effectiveness of our method.

\section{Acknowledgments}
This work is supported in part by the US National Science Foundation under grants IIS-2106961, and IIS-2144338. Any opinions, findings, and conclusions or recommendations expressed in this material are those of the author(s) and do not necessarily reflect the views of the National Science Foundation.

\bibliography{main}

\clearpage

\section{Appendix}
\subsection{Loss Functions for the Adopted Self-explaining Neural Networks}
In this section, we detail the loss functions for the adopted self-explaining neural networks (concept bottleneck models \cite{koh2020concept,havasi2022addressing,chauhan2023interactive,yuksekgonul2022post}, and prototype-based self-explaining networks \cite{zhang2023learning,alvarez2018towards, huai2022towards, li2018deep}).

\textbf{Loss functions for concept bottleneck models.} Concept bottleneck models (CBMs) are interpretable neural networks that first predict labels for human-interpretable concepts relevant to the prediction task, and then predict the final label based on the concept label predictions \cite{koh2020concept,havasi2022addressing,chauhan2023interactive}. We denote raw input features by $x$, the labels by $y$, and the pre-defined true concepts by $c \in \mathbb{R}^{C}$. Given the observed training samples $D^{tra}=\{(x_i,c_i,y_i)\}_{i=1}^{N^{tra}}$, a concept bottleneck model \cite{koh2020concept} can be trained through using the concept and class prediction losses via two distinct training approaches: \emph{independent} and \emph{joint}. The two training paradigms are briefly discussed below.
\begin{itemize}
    \item \emph{Joint training:} This joint training strategy learns both the concept and prediction models ($h$ and $g$) by minimizing both concept and classification losses jointly in an end-to-end manner. Mathematically, it aims to minimize the following overall training objective
    \begin{align}
        & L_{g,h}= \sum_{i=1}^{N^{tra}} [\ell (g(h(x_{i});y_{i})) \notag \\
        &\qquad\qquad + \xi_{1} * \sum_{j=1}^{C} \ell ({h}_{j}(x_{i});c_{i,j})],
    \end{align}
    where $\xi_{1}$ is a pre-defined parameter, $N^{tra}$ is the number of training samples. In the above, the first and second terms represent the classification and concept losses for the $i$-th training sample with $C$ total number of concepts, respectively. In our experiments, we set $\xi_1=1$.

    \item \emph{Independent training:} This independent training strategy learns the concept predictor $h$ and the class prediction component $g$ independently as follows 
    \begin{align}
        &\quad g = \arg\min_{g} \sum_{i=1}^{N^{tra}} \ell( g(c_i); y_{i}), \\
        & h = \arg\min_{h} \sum_{i=1}^{N^{tra}} \sum_{j=1}^{C} \ell ({h}_{j}(x_{i});c_{i,j}),
    \end{align}
    where $N^{tra}$ denotes the number of training samples, $g$ is trained using the true concepts during the training stage. Note that during the test stage, for a given test sample $x$, the final prediction function $g$ still takes $h(x)$ as its input.
    
\end{itemize}
Note that for concept bottleneck models \cite{koh2020concept,havasi2022addressing,chauhan2023interactive}, the standard training strategy ignores the concepts and directly maps input features to the final classification prediction.

\textbf{Loss functions for prototype-based self-explaining neural networks.} Different from concept bottleneck models \cite{koh2020concept,havasi2022addressing,chauhan2023interactive}, prototype-based self-explaining neural networks usually directly learn a set of prototype-based concepts directly from the training data (i.e., $D^{tra}=\{(x_i,y_i)\}_{i=1}^{N^{tra}}$) during the training process without the pre-defined concept supervision information. These prototype-based concepts are extracted in a way that they can best represent specific target sets \cite{zhang2023learning}. Specifically, they usually adopt the autoencoder network to learn a lower-dimension latent representation of the data with an encoder network, and then use several network layers over the latent space to learn $C$ prototype-based concepts, i.e., $\{p_{j}\in \mathbb{R}^{q}\}_{j=1}^{C}$. After that, for each $p_{j}$, the similarity layer computes its distance from the learned latent representation (i.e., $e_{enc}(x)$) as $h_{j}(x)=|| e_{enc}(x)-p_{j} ||_{2}^{2}$. The smaller the distance value is, the more similar $e_{enc}(x)$ and the $j$-th prototype ($p_{j}$) are. Finally, we can output a probability distribution over the $M$ classes. Let $R(x,\tilde{x})$ denote the reconstruction loss (e.g., $R(x,\tilde{x})=||x-\tilde{x}||_{2}$ that calculates the distance between the original and reconstructed input for penalizing the autoencoder’s reconstruction error.). To obtain the prototype-based self-explaining neural networks \cite{li2018deep}, we can minimize the following loss
\begin{align}
    & \sum_{i=1}^{N^{tra}} \ell (g(h(x_i));y_{i}) + \xi_{2} R(x_{i},\tilde{x}_{i}) \notag \\
    &\quad + \xi_{3} E_{1} (\{p_{j}\in \mathbb{R}^{q}\}_{j=1}^{C}) + \xi_{4} E_{2} (\{p_{j}\in \mathbb{R}^{q}\}_{j=1}^{C}), \notag
\end{align}
where $\xi_{2}$ is a pre-defined trade-off parameter, and $E_1$ and $E_{2}$ are two interpretability regularization terms. In \cite{li2018deep}, the minimization of $E_1$ enforces that each prototype-based concept be as close as possible to at least one of the training examples in the latent representation space, and the minimization of $E_2$ would require every encoded training example to be as close as possible to one of the prototype-based concepts. In the above, $\sum_{i=1}^{N^{tra}} R(x_{i},\tilde{x}_{i})$ is the autoencoder's reconstruction error on the given training dataset. In our experiments, we set $\xi_{2}=1$, $\xi_{3}=1$, and $\xi_{4}=1$.

\subsection{Proof of Theorem 3}
\begin{manualtheorem} {3} \label{thm:LabelSetThm3}
    Suppose the calibration samples and the test sample are exchangeable, if we construct $ \Gamma^{\varepsilon}_{lab}(x^{test})$ as indicated above, the following inequality holds for any $\varepsilon \in (0,1)$
    \begin{align}
    & P(y^{test} \in \Gamma^{\varepsilon}_{lab} (x^{test}) ) \geq 1-\varepsilon,
    \end{align}
    where $y^{test}$ is the true label of $x^{test}$, and $\Gamma^{\varepsilon}_{lab} (x^{test})$ is the label prediction set derived based on Eqn. (7) in the main manuscript.
\end{manualtheorem}

\begin{proof}
    To simplify notations, in the following, we use 
    $\tilde{D}=\{(x_i,c_i,y_i)\}_{i=1}^{N^{tra}} \cup \{(x^{test},c^{test},y^{test})\}$, where $N^{tra}$ is the number of training samples. Below, we will prove that for any given function $B: \mathcal{X} \times \mathcal{C} \rightarrow \mathbb{R}$ that is independent of data points in $\tilde{D}$, we have that $B(x_{i},c_{i})$ are exchangeable. Let $\tau$ denote the random permutation of the integers $1,\cdots,N^{tra}, (N^{tra}+1)$. Below, we use $C_{B^{-1}}(u-)$ to denote the set $\{ (x,c): B(x,c) \leq u \}$ (i.e., $C_{B^{-1}}(u-) = \{ (x,c): B(x,c) \leq u \}$). Then, we can have
    \begin{align}
        & F_{v} (u_{1},\cdots,u_{\tilde{N}} | D_{tra}) \notag \\
        & = P(B(x_{1},c_1)\leq u_1, \cdots, B(x_{\tilde{N}},c_{\tilde{N}})\leq u_{\tilde{N}} | D_{tra} ) \notag \\
        &= P ((x_{1},c_{1}) \in C_{B^{-1}}(u_{1}-), \cdots, (x_{\tilde{N}},c_{\tilde{N}}) \notag \\
        & \qquad \qquad \in C_{B^{-1}}(u_{\tilde{N}}-) | D_{tra} ) \notag \\
        & =P ((x_{\tau(1)},c_{\tau(1)}) \in C_{B^{-1}}(u_{1}-), \cdots, (x_{\tau(\tilde{N})},c_{\tau(\tilde{N})}) \notag \\
        & \qquad \qquad \in C_{B^{-1}}(u_{\tilde{N}}-) | D_{tra} ) \notag \\
        &= P (B (x_{\tau(1)},c_{\tau(1)}) \leq u_{1}, \cdots, B (x_{\tau(\tilde{N})}, \notag \\
        & \qquad \qquad c_{\tau(\tilde{N})}) \leq u_{\tilde{N}}  | D_{tra} ) \notag \\
        & = F^{\tau}_{v} (u_{1},\cdots, u_{\tilde{N}}  | D_{tra} ),
    \end{align}
    where $\tilde{N}=N^{tra}+1$, $F_{v} (u_{1},\cdots,u_{\tilde{N}} | D_{tra})$ is the CDF, $\tau$ denotes a random perturbation, $F^{\tau}_{v} (u_{1},\cdots, u_{\tilde{N}}  | D_{tra} )$ denotes the perturbation CDF, and $C_{h^{-1}}(u-) = \{ (x,c): B(x,c) \leq u \}$.

    From the above, we can know that the proposed non-conformity score function is independent of the dataset $\tilde{D}$. The reason is that in the proposed non-conformity score function, we only use the information of $g$ and $h$ that is dependent on the training set $D^{tra}$ instead of $\tilde{D}$. In addition, when calculating the non-conformity score for each sample originating from $\tilde{D}$, there is no requirement to access information from the calibration samples for any other instances.

    Based on the above, we can conclude that the non-conformity scores on $\tilde{D}$ are exchangeable. Based on the Lemma 2 in the main manuscript, we have that the theoretical guarantee (i.e., $P(y^{test} \in \Gamma^{\varepsilon}_{lab} (x^{test}) ) \geq 1-\varepsilon$) holds for any $\varepsilon \in (0,1)$.    
\end{proof}

\subsection{Discussions on Prototype-based Self-explaining Neural Networks}
In this section, we will give more details about how to generalize our above method to the prototype-based self-explaining networks, which do not depend on the pre-defined expert concepts and directly learn a set of prototype-based concepts from the training data during training \cite{alvarez2018towards,li2018deep}. They usually adopt the autoencoder network to learn a lower-dimension latent representation of the data with an encoder network, and then use several network layers over the latent space to learn $C$ prototype-based concepts, i.e., $\{p_{j}\in \mathbb{R}^{q}\}_{j=1}^{C}$. After that, for each $p_{j}$, the similarity layer computes its distance from the learned latent representation (i.e., $e_{enc}(x)$) as $h_{j}(x)=|| e_{enc}(x)-p_{j} ||_{2}^{2}$. The smaller the distance value is, the more similar $e_{enc}(x)$ and the $j$-th prototype ($p_{j}$) are. Finally, we can output a probability distribution over the $M$ classes. Nonetheless, we have no access to the true basis explanations if we want to conduct conformal predictions at the prototype-based interpretation level. To address this, for prototype-based self-explaining networks, we construct the following non-conformity measure
\begin{align}
\label{eq:ApNonconforScoreFAutoEn}
    & s(x_i,y_{i},f)=  \inf_{v \in \{v: \ell(g(v),y_{i}) \leq \alpha\} } [\sum_{j=1}^{K} |h_{[j]}(x_{i})-v_{[j]}| \notag \\
    &\qquad\qquad + \lambda_{2} * \sum_{j=K+1}^{C} |h_{[j]}(x_{i})-v_{[j]}| ],
\end{align}
\noindent where $f=g \circ h$, $h(x_{i}) \in \mathbb{R}^{C}$ denotes the similarity vector for $x_{i}$, $\ell(g(v),y_{i})$ is the classification cross-entropy loss for $v$, and $\alpha$ is a predefined small threshold value to make the prediction based on $v$ to be close to label $y_i$. 

Without loss of generality, in the following, we consider the case where $\lambda_{2}=1$. In this case, we can rewrite the above non-conformity measure as
\begin{align}
    & s(x_i,y_{i},f)= \inf_{v \in \{v: \ell(g(v),y_{i}) \leq \alpha\} } ||v-h(x_{i})||_{1},
\end{align}
Note that the extension to other cases can be straightforwardly derived. Based on the above non-conformity measure, we can calculate the non-conformity scores for the calibration samples $D^{cal}$, and then calculate the desired quantile value $Q_{1-\varepsilon}$. In Algorithm \ref{alg:NonconfScores}, we give the detailed procedure for calculating an upper bound of the non-conformity score for a given sample $(x_i,y_i)$.

\begin{algorithm}[!t]
    \caption{Non-conformity score}
    \label{alg:NonconfScores}
    \begin{algorithmic}[1]
        \STATE \textbf{Input:} Data sample $(x_i,y_i)$, pre-trained model $f=g \circ h$, step size $\eta$, and number of steps $O$.
        \STATE \textbf{Output:} The non-conformity score $s(x_i,y_i,f)$.
        \STATE $u \leftarrow h(x_i)$
        \STATE $o \leftarrow 0$
        \WHILE{$o < O$} 
            \STATE $u \leftarrow u - \eta \frac{\partial ||\ell(g(u),y_i) -\alpha||_{1}}{\partial u}$
            \STATE $o \leftarrow o+1$
        \ENDWHILE
        \STATE {\bfseries Return:} $s(x_i,y_i,f)=||u-h(x_{i})||_{1}$.
    \end{algorithmic}
\end{algorithm}

Then, for the given test sample $x^{test}$, we calculate its final prediction set for $x^{test}$ as
\begin{align}
    & \Gamma^{\varepsilon}_{lab} (x^{test})  = \{y=g(v): ||v-h(x^{test})||_{1} \leq Q_{1-\varepsilon}\}, \notag 
\end{align}
\noindent where the quantile value $Q_{1-\varepsilon}$ is derived based on the non-conformity measure in Eqn. (\ref{eq:ApNonconforScoreFAutoEn}). To derive the label sets, we propose to solve the following optimization
\begin{align}
\label{eq:ApLabelSetOpt22}
    & \quad \max_{v \in \mathbb{R}^{C}} \sum_{m' \in [M]\setminus m} \mathbb{I} [g_{m}(v) > g_{m'} (v)] \\
    &\quad\quad s.t., ||v-h(x^{test})||_{1} \leq Q_{1-\varepsilon},\notag 
\end{align}
\noindent where the quantile value $Q_{1-\varepsilon}$ is derived based on Eqn. (\ref{eq:ApNonconforScoreFAutoEn}). In the above, we want to find a $v$, whose predicted label is ``$m$'', under the constraint that $||v-h(x^{test})||_{1} \leq Q_{1-\varepsilon}$. If we can successfully find such a $v$, we will include the label ``$m$'' in the label set; otherwise, we will exclude it. Notably, we need to iteratively solve the above optimization over all possible labels to construct the final prediction set.

\begin{table*}[htbp] 
\centering
\begin{tabular}{c||c||c} 
\toprule
Label ID & Super-class (coarse label) & Concept (sparse label) \\ 
\midrule
0 & aquatic\_mammals & beaver, dolphin, otter, seal, whale \\
1 & fish & aquarium\_fish, flatfish, ray, shark, trout \\
2 & flowers & orchid, poppy, rose, sunflower, tulip \\
3 & food\_containers & bottle, bowl, can, cup, plate\\
4 & fruit\_and\_vegetables & apple, mushroom, orange, pear, sweet\_pepper \\
5 & household\_electrical\_devices & clock, keyboard, lamp, telephone, television \\
6 & household\_furniture & bed, chair, couch, table, wardrobe \\
7 & insects & bee, beetle, butterfly, caterpillar, cockroach \\
8 & large\_carnivores & bear, leopard, lion, tiger, wolf \\
9 & large\_man-made\_outdoor\_things & bridge, castle, house, road, skyscraper \\
10 & large\_natural\_outdoor\_scenes & cloud, forest, mountain, plain, sea \\
11 & large\_omnivores\_and\_herbivores & camel, cattle, chimpanzee, elephant, kangaroo \\
12 & medium\_mammals & fox, porcupine, possum, raccoon, skunk \\
13 & non-insect\_invertebrates & crab, lobster, snail, spider, worm \\
14 &people & baby, boy, girl, man, woman \\
15 & reptiles & crocodile, dinosaur, lizard, snake, turtle \\
16 & small\_mammals & hamster, mouse, rabbit, shrew, squirrel \\
17 & trees & maple\_tree, oak\_tree, palm\_tree, pine\_tree, willow\_tree \\
18 & vehicles\_1 & bicycle, bus, motorcycle, pickup\_truck, train \\
19 & vehicles\_2 & lawn\_mower, rocket, streetcar, tank, tractor \\
\bottomrule
\end{tabular}
\caption{CIFAR-100 Super-class label mapping. The 100 classes in the CIFAR-100 are grouped into 20 super-classes.}
\label{tab:cifar100_super-class}
\end{table*}

However, it is difficult to directly optimize the above problem, due to the non-differentiability of the discrete components in the above equation. Below, we use $\delta=[\delta_{1},\cdots,\delta_{C}]$ to denote the difference between $v$ and $h(x^{test})$, i.e., $\delta =v-h(x^{test})$. Based on this, we can have
\begin{align}
\label{eq:ApVTransform}
    & v=[v_{1}=\delta_{1}+h_{1}(x^{test}),\cdots, \\
    &\qquad\quad v_{C}=\delta_{C}+h_{C}(x^{test})]. \notag
\end{align}
\noindent To address the above challenge, we propose to solve the below reformulated optimization
\begin{align}
\label{eq:ApDeltaOptima}
  &  \min_{||\delta ||_{1} \leq Q_{1-\varepsilon}} \max (\max_{m'\neq m} g_{m'}(v) - g_{m}(v),-\beta),
\end{align}
\noindent where $m' \in [M]$, $\beta$ is a pre-defined value, $v$ is derived based on Eqn. (\ref{eq:ApVTransform}), and $g: \mathbb{R}^{C} \rightarrow \mathbb{R}^{M}$ maps $v$ into a final class prediction. $g(v)$ represents the logit output of the underlying self-explaining model for each class when the input is $v$. The above adversarial loss enforces the actual prediction of $v$ to the targeted label ``$m$'' \cite{huang2020metapoison,carlini2017towards}. If we can find such a $v$, the label ``$m$'' will be included in the prediction set; otherwise, it will not be included in the final prediction sets.

\subsection{More Experimental Details}
\textbf{Real-world datasets.} In experiments, we adopt the following real-world datasets: \textbf{CIFAR-100 Super-class} \cite{fischer2019dl2} and \textbf{MNIST} \cite{deng2012mnist}.
\begin{itemize}
    \item The CIFAR-100 Super-class dataset is a variant of the CIFAR-100 \cite{krizhevsky2009learning} image dataset. The CIFAR-100 dataset consists of 60,000 color images, each with dimensions of $32\times32$ pixels, distributed across 100 distinct classes. Each class has 600 images, with 500 training images and 100 testing images. In the super-class version, the 100 classes of images are further grouped into 20 super-classes. For instance, the five classes \emph{baby, boy, girl, man} and \emph{woman} belong to the super-class \emph{people}. In this case, we exploit each image class as a concept-based explanation for the super-class prediction \cite{hong2023concept}. The detailed information about super-classes and the concepts can be found in Table \ref{tab:cifar100_super-class}.
    
    \item The MNIST dataset consists of a collection of grayscale images of handwritten digits (0-9). Each digit is represented as a 28$\times$28 pixel image. Each image is a small square with pixel values ranging from 0 to 255, indicating the intensity of the grayscale color. There are 60,000 training images and 10,000 testing images in the dataset. In experiments, we consider the MNIST range classification task, wherein the goal is to classify handwritten digits into five distinct classes based on their numerical ranges. Specifically, the digits are divided into five non-overlapping ranges (i.e., $[0,1], [2,3], [4,5],[6,7], [8,9]$), each assigned a unique label (the label set is denoted as $\{0, 1, 2, 3, 4\}$). For example, an input image depicting a handwritten digit ``4'' is classified as class 2 within the range of digits 4 to 5. In this case, we exploit the identity of each digit as a concept-based explanation for the ranges prediction \cite{rigotti2021attention}. 
\end{itemize}

\begin{table}[t] 
\small
\centering
\begin{tabular}{ccc} 
\toprule
Dataset & Calibration set size (\%) & Running time (s) \\ 
\midrule
\multirow{4}{*}{MNIST} & 5 & $0.77 \pm 0.01 $\\
& 10 & $1.52 \pm 0.01$\\
& 15 & $2.28 \pm 0.01$\\
& 20 & $3.04 \pm 0.01$\\
\midrule
\multirow{4}{*}{\makecell{CIFAR-100 \\ Super-class}} & 5 & $1.25 \pm 0.04$\\
& 10 & $2.41 \pm 0.02$\\
& 15 & $3.60 \pm 0.04$\\
& 20 & $4.78 \pm 0.04$\\
\bottomrule
\end{tabular}
\caption{Running time for computing non-conformity score on different calibration set sizes (splitting percentage of the original available data).}
\label{tab:time_score}
\end{table}

\begin{table}[t] 
\small
\centering
\begin{tabular}{ccc} 
\toprule
Dataset & Test set size (\%) & Running time (s) \\ 
\midrule
\multirow{5}{*}{MNIST} & 20 & $2.35 \pm 0.02 $\\
& 40 & $4.68 \pm 0.05$\\
& 60 & $7.06 \pm 0.06$\\
& 80 & $9.38 \pm 0.09$\\
& 100 & $11.75 \pm 0.12$\\
\midrule
\multirow{5}{*}{\makecell{CIFAR-100 \\ Super-class}} & 20 & $21.03 \pm 0.23$\\
& 40 & $41.91 \pm 0.39$\\
& 60 & $62.81 \pm 0.66$\\
& 80 & $83.73 \pm 0.83$\\
& 100 & $104.74 \pm 1.09$\\
\bottomrule
\end{tabular}
\caption{Running time for computing concept conformal sets with different test set sizes.}
\label{tab:time_concept_set}
\end{table}

\textbf{Parameter settings.} In experiments, we apply various machine learning models to evaluate our proposed method on the adopted real-world datasets. Specifically, for the CIFAR-100 Super-class dataset, we train a self-explaining model that leverages ResNet-50 to learn the underlying concepts and a multi-layer perception (MLP) to learn the labels. The training process involves an SGD optimizer with an initial learning rate of 0.1, a batch size of 128, and a total of 300 epochs. The learning rate is further decayed by $10\times$ at the 150-epoch and 225-epoch marks. For the MNIST dataset, we train a self-explaining model that uses a convolutional neural network (CNN) to learn the concepts and an MLP to learn the labels. The CNN architecture incorporates two convolutional layers with 16 and 32 kernels of size 5. These layers are complemented by ReLU activation functions and max-pooling layers. The training process also involves an SGD optimizer with a learning rate of 0.1, a batch size of 128, and a total of 100 epochs. In our experimental setup, we randomly split the original available dataset into $90\%$ as the training set and $10\%$ as the calibration set. We consider the top-ranked concepts and set $\lambda_2 = 0$ to compute the non-conformity scores. We repeat the experiments 10 times and report the means and standard errors of the results.

\textbf{Computing infrastructure.} The experiments are implemented using the PyTorch framework and run on a Linux server. The Linux server configures a GPU machine with AMD EPYC 7513 32-core 2.6GHz CPUs and Nvidia A100 PCIe 40GB GPUs.

\subsection{More Experimental Results}

\textbf{Running time analysis.} Here, we provide a running time analysis of the self-explaining model we employed in the main manuscript. In Table \ref{tab:time_score}, we report the running time for computing the non-conformity score on different calibration set sizes. Literally, the running time to compute the non-conformity score is related to the size of the calibration set. Larger calibration sets take more computation time. In Table \ref{tab:time_concept_set}, we report the running time for computing the concept conformal sets with different test set sizes. Overall, the computational cost of our proposed method is reasonable.

\begin{figure}[t]
\centering
\begin{subfigure}{0.7\linewidth}
\includegraphics[width=1\linewidth]{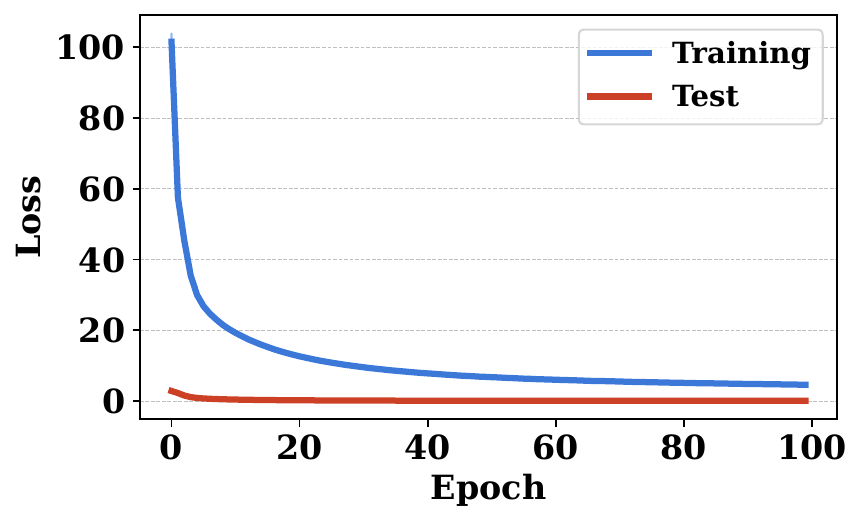}
\caption{Loss}
\end{subfigure}
\hfill
\begin{subfigure}{0.7\linewidth}
\includegraphics[width=1\linewidth]{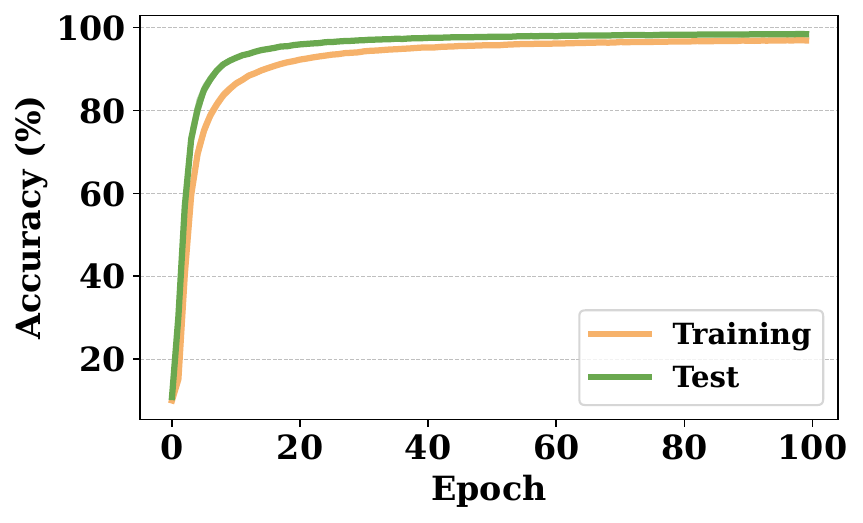}
\caption{Accuracy}
\end{subfigure}
\caption{Training and test performance of the prototype-based self-explaining model on MNIST.}
\label{fig:training_test}
\end{figure}

\textbf{More experimental results on prototype-based self-explaining neural networks.} First, we explore the training and test performance of the prototype-based self-explaining model on the MNIST dataset. The prototype-based self-explaining model consists of an autoencoder with four convolutional layers in both the encoder and decoder. The four convolutional layers in the encoder each use a kernel of size 3. Each output of the first three layers has 32 feature maps, while the last layer has 10 feature maps. We train the model for a total of 100 epochs using a learning rate of 0.001. Figure \ref{fig:training_test} shows the experimental results in terms of training loss, test loss, training accuracy, and test accuracy. Evidently, the training loss and test loss exhibit a consistent downward trend as the epochs progress, eventually converging at an early stage. After training for 100 epochs, the model achieves $96.9\%$ accuracy on the MNIST training set and $98.4\%$ accuracy on the MNIST test set. Next, we demonstrate the non-conformity score distribution of the prototype-based self-explaining model on the MNIST dataset. Here, we randomly split $10\%$ of the original available dataset as the calibration set. We choose a miscoverage rate $\varepsilon=0.1$. The outcome is visualized in Figure \ref{fig:cal_score}. As we can see, the non-conformity scores follow a normal distribution, which implies that the non-conformity measure is reasonable. Then, in Table \ref{tab:coverage_mnist}, we present the label coverage (the percentage of times the prediction set size is 1 and contains the true label) and average set size for the prototype-based self-explaining model on the MNIST dataset. The results reveal that our proposed method (i.e., unSENN) consistently outperforms both the Naive and Vanilla baselines. For example, our proposed method achieves a label coverage of 0.984 with an average set size of 1.00, while the Naive baseline only reaches a label coverage of 0.922 with an average set size of 1.11. Overall, we can find that our proposed method is effective in producing uncertainties for the final predictions on prototype-based self-explaining neural networks, leveraging the informative prototype-based concepts.

\begin{figure}[t]
\centering
\includegraphics[width=0.7\linewidth]{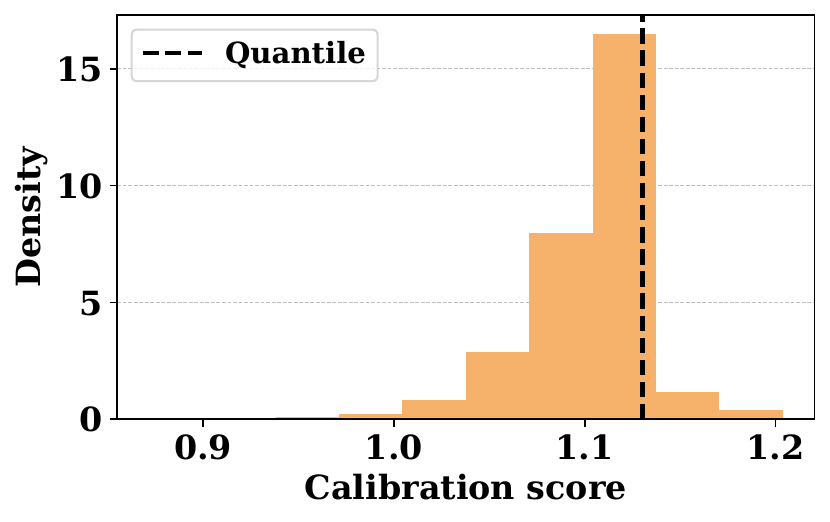}
\caption{Non-conformity score distribution of the prototype-based self-explaining model on MNIST.}
\label{fig:cal_score}
\end{figure}

\begin{table}[t] 
\centering
\begin{tabular}{ccc} 
\toprule
Method & Label coverage & Average set size \\ 
\midrule
unSENN & $0.984 \pm 0.001 $ & $1.00 \pm 0.00$\\
Naive & $0.922 \pm 0.002$ & $1.11 \pm 0.00$\\
Vanilla & $0.908 \pm 0.002$ & $0.91 \pm 0.00$\\
\bottomrule
\end{tabular}
\caption{Label coverage and average set size of the prototype-based self-explaining model on MNIST. }
\label{tab:coverage_mnist}
\end{table}

\end{document}